\documentclass[11pt,letterpaper]{article}

\usepackage{amsmath,amssymb,amsthm,epsfig}
\usepackage{graphpap}
\usepackage{color}
\usepackage{graphicx}
\usepackage{epsfig}
\usepackage{amsmath}
\usepackage{latexsym}
\usepackage{amsfonts}
\usepackage{bm}
\usepackage{pdfpages}
\usepackage{authblk}

\newtheorem{theorem}{Theorem}
\newtheorem{lemma}{Lemma}
\newtheorem{definition}{Definition}

\newcommand{\cR}{\mathbb{R}}

\usepackage[pdftex]{hyperref}

\title{Notes on Worst-case Inefficiency of Gradient Descent Even in $\cR^2$}
\author[1]{Shiliang Zuo}
\affil[1]{University of Illinois at Urbana-Champaign}
\date{}
\begin{document}
	\maketitle
	
\begin{abstract}
    Gradient descent is a popular algorithm in optimization, and its performance in convex settings is mostly well understood. In non-convex settings, it has been shown that gradient descent is able to escape saddle points asymptotically and converge to local minimizers~\cite{lee2016gradientminimizers}. Recent studies also show a perturbed version of gradient descent is enough to escape saddle points efficiently~\cite{Jin:EscapeSaddle,ge2015escaping}. In this paper we show a negative result: gradient descent may take exponential time to escape saddle points, with non-pathological two dimensional functions. While our focus is theoretical, we also conduct experiments verifying our theoretical result. Through our analysis we demonstrate that stochasticity is essential to escape saddle points efficiently. At the time of writing, this technical report was the product of a course research project of the author. 
\end{abstract}

\section{Introduction} 
With the advance of machine learning and deep learning, understanding how different optimization methods affect the performance of machine learning systems, either theoretically or empirically, has attracted much interest recently. Classical theory of optimization methods (e.g. gradient descent or stochastic gradient descent), mostly focus on convex problems (for an overview, see e.g.~\cite{bubeck2014convex}). However, in most machine learning problems, the function at hand can be highly non-convex, which calls for new analysis of classical optimization methods (for an overview of non-convex optimization, see e.g.~\cite{Jain_2017}). Non-convex functions may admit stationary points which are not global minima, but instead are saddle points and local extremas, as opposed to convex functions where any stationary point is a global minima. 

Gradient Descent is a classical optimization method, and has gained popularity in a wide range of applications, including machine learning, big data optimization, etc. Different variations of gradient descent has also developed through the years, with most notable ones being heavy ball and Nestorov's momentum~\cite{nesterov1983method} (sometimes also referred to as accelerated gradient descent). The convergence behavior of gradient descent is mostly understood in convex problems, but we currently have far less understanding of gradient descent for non-convex problems. In many applications, finding a local minima is often enough, and recent studies have shown that for some problems, spurious local minima doesn't exist, meaning all local minimas are global~\cite{ge2017nospuriouslocalminima}. 

However, this isn't to say applying gradient descent in non-convex problems doesn't have any issues. Due to the existence of saddle points, gradient descent may get slowed down, and may even get stuck at saddle points. In fact, it is fairly easy to construct some function with an artificial initialization scheme, such that gradient descent will take an exponential number of iterates to escape a saddle point. Du et. al.~\cite{du2017gradient} considers the function $f(x) = x_1^2 - x_2^2$. If the initial point is in an exponential thin band along the $x_1$ axis centered at the origin, then gradient descent will take exponential iterates to escape the neighborhood of the saddle point at $(0, 0)$. 

In this paper, we ask the following, does there exist some non-pathological function with a fairly reasonable initialization scheme, such that gradient descent is inefficient when escaping saddle points? The answer is positive, even if the function is 2-dimensional. Our result show the worst case inefficiency of gradient descent when applied to non-convex problems.

\subsection{Related Work}
In non-convex problems, saddle points slow down gradient descent, and convergence speed is slower than its convex counterparts. Nontheless, it is shown that gradient descent is able to escape saddle points, and converge to local minimizers asymptotically, under a random initialization scheme~\cite{lee2016gradientminimizers}. Further, Ge et. al. and Jin et. al. showed a perturbed version of gradient descent is enough to escape saddle points efficiently~\cite{Jin:EscapeSaddle, ge2015escaping}. 
Their modified version of gradient descent is based on an exploitation-exploration trade-off: when the gradient is large, standard gradient descent is applied, thus exploiting larger gradients; when the gradient is small, a certain randomness is added to the descent direction, thus exploring in hope to escape saddle points. Researchers also analyzed the landscape of different classes of functions, and Ge et. al. showed for matrix completion, no spurious local minima exists, that is all local minimizers are necessarily global~\cite{ge2017nospuriouslocalminima, ge2016matrix}. 

On negative results, a work mostly related to our paper is by Du et. al.~\cite{du2017gradient}. They showed that for some non-pathological $d$ dimensional function with $d$ saddle points, gradient descent takes exponential time in $d$ to converge to a local minima. 

\subsection{Summary of Results}

In this work, we show a negative result related to the performance of gradient descent on non-convex functions. We construct a two-dimensional smooth function which gradient descent takes exponential time in the number of saddle points to find the global minima. 

Our paper is organized as follows. Section 2 gives a high-level idea of our construction and intuition on why gradient descent fails to escape saddle points efficiently. Section 3 gives formal mathematical proofs. Section 4 contains our empirical studies verifying our theoretical finding. 

\section{Main Result and Proof Idea}
In this section we give a high-level description of our objective function, and show how vanilla gradient descent fails to escape saddle points efficiently. 
We adopt the classical definition of saddle points. 
\begin{definition}
A saddle point is a stationary point that is neither a local minima nor a local maxima. 
\end{definition}
\subsection{High-level Construction of Objective Function}

We will give the construction of our function in 5 steps. We remark that the crux of our construction are in the first two steps. Step 3 through 5 essentially makes the construction more rigorous. 

\textbf{Step 1. Dividing Blocks. }Define 
\begin{align*}
B_1 &= [0, 1] \times [0, 1]\\
B_2 &= [1,2] \times [0,1] \\
B_i &= B_{i-2} + (1, 1)
\end{align*}
In other words, $B_i$ will be $B_{i-2}$ translated to the right by 1 unit then up by 1 unit. 

Define the region $D = \bigcup_{i=1}^n B_i$. We will refer to $B_i$ as the $i$-th block. The nature of $B_n$ will be different from other blocks, and will be referred to as the final block. Our function will be defined on the region $D$, and we will later extend the domain to $\cR^2$. 

\textbf{Step 2. Constructing saddle points in each block. } In each block, the function will be quadratic with a saddle point in the center. In odd blocks the saddle point will be a local maxima along $x_1$-axis, and a local minima along $x_2$-axis. In even blocks the saddle point will be a local maxima along $x_2$-axis, and a local minima along $x_1$-axis. Specifically, in odd blocks the function is defined as
\begin{align*}
    f(x_1, x_2) = -\gamma (x_1 - s_1)^2 + L (x_2 - s_2)^2
\end{align*}
and in even blocks
\begin{align*}
    f(x_1, x_2) = L (x_1 - s_1)^2 - \gamma (x_2 - s_2)^2
\end{align*}
where $(s_1, s_2)$ is the center of the respective block. 
Finally, a local minima will occur in the final block. The function inside the final block is defined as
\begin{align*}
    f(x_1, x_2) = L (x_1 - s_1)^2 + L (x_2 - s_2)^2
\end{align*}

\textbf{Step 3. Flipping the quadratic. }
When running gradient descent with initial point inside the first block, we want the iterates to follow along the blocks. That is, in odd blocks, iterates would go toward the right to the next block, and in even blocks, iterates would go upward to the next block. However, it is quite possible in some iterate the point lies on the `wrong side of the block'. For odd blocks, this is the left half of itself and for even blocks this is the bottom half. For example, if some point were in the left half of some odd block, gradient descent iterate would make the next iterate go toward the left. To combat this, we simply flip the quadratic on the left half of odd blocks and the bottom half of even blocks.

\textbf{Step 4. Adding buffer regions. } Up until step 3, our function is not continuous. To resolve this issue we add buffer regions between adjacent blocks. Our domain $D$ will have to be modified to include the buffer regions. 


\textbf{Step 5. Extending to $\cR^2$. }Our current construction is defined on a restricted region. To extend our function to $\cR^2$, we apply the Whitney extension theorem. The extension may introduce new stationary points, but we will show gradient descent never leaves region $D$. 

\subsection{Exponential Time needed to Escape Saddle Points}
Let $f$ be the function outlined previously. We will show the high level ideas that the time needed for gradient descent to converge to the unique local minima is exponential in the number of blocks $n$. For simplicity, we will first consider the behavior of gradient descent running on the function constructed up until step 2. 

Consider some even block. Let $d_2$ be the distance between $x_2$ and the center line along $x_1$ axis: $d_2 = |x_2 - s_2|$. In this block $x_2$ is trying to escape to the next block, with $d_2$ growing by a rate of $(1 + 2\eta\gamma)$. In the previous block $x_2$ is converging to the center with rate $(1 - 2\eta L)$, while $x_1$ is trying to escape. Suppose $t$ iterations is needed to escape the previous block, and $t'$ iteration is needed to escape the current block. Then $d_2$ is no larger than $(1 - 2\eta L) ^ t$ when $x_1$ successfully escapes the previous block. Thus we must have
\[
    (1 - 2\eta L)^t (1 + 2\eta \gamma)^{t'} \ge 1
\]
Taking logarithms and by simple algebra we can lower bound $t'$ as
\[
    t' \ge \frac{L}{\gamma} t
\]
Thus the number of iterations needed to escape saddle points is growing by a multiplicative factor each time. This essentially shows in order to escape $n$ saddle points, we need at least $\Omega((\frac{L}{\gamma})^n)$ iterations. 

When step 3 through 5 is added into consideration, more careful treatment is needed, and we defer the formal proof to the next section. 

\begin{figure}
    \centering
    \includegraphics[scale = 0.5]{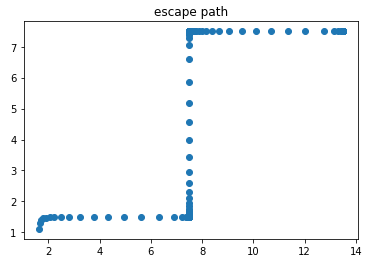}
    \caption{Escape Path of Gradient Descent}
    \label{fig:escape_path}
\end{figure}

\section{Proofs}

\subsection{Detailed Construction}
In our construction, $L$ and $\gamma$ will be fixed constants with $L \ge \gamma$, and let $L_2 = 4L$. $[x]_{2\tau}$ will denote the largest number no larger than $x$ that is a multiple of $2\tau$. We will use $(s_1^i, s_2^i)$ to denote the center of a block $i$, though we will usually omit the superscript as the block in reference is usually clear from context. 

Define 
\begin{align*}
    B_1 = [0, \tau] \times [0, \tau],&\quad B'_1 = [\tau , 2\tau] \times [0, \tau] \\
    B_2 = [2\tau, 3\tau] \times [0,  \tau],&\quad B'_2 = [2\tau, 3\tau] \times [\tau, 2\tau]\\
    B_i = B_{i-2} + (2\tau, 2\tau),&\quad B'_i = B'_{i-2} + (2\tau, 2\tau)
\end{align*}
We will refer to $B_i$ with $i$ odd as odd blocks, with $i$ even as even blocks; $B'_i$ with odd $i$ as odd-to-even buffers, with $i$ even as even-to-odd buffers. In the following $\nu$ will be some constant chosen later.

If $(x_1, x_2) \in B_i$ and $i$ is odd:
\begin{align*}
    f(x_1, x_2) = -i\nu +
    \begin{cases} -\gamma (x_1 - s_1)^2 + L (x_2 - s_2)^2 &\mbox{if } x_1 > s_1 \\
                L_2 (x_1 - s_1)^2 + L (x_2 - s_2)^2  & \mbox{if } x_1 \le s_1
    \end{cases}
\end{align*}
If $(x_1, x_2) \in B_i$ and $i$ is even:
\begin{align*}
    f(x_1, x_2) = -i\nu +
    \begin{cases} L (x_1 - s_1)^2 - \gamma (x_2 - s_2)^2 &\mbox{if } x_2 > s_2 \\
                L (x_1 - s_1)^2 + L_2 (x_2 - s_2)^2  & \mbox{if } x_2 \le s_2
    \end{cases}
\end{align*}
If $(x_1, x_2) \in B'_i$ and $i$ is odd:
\begin{align*}
    f(x_1, x_2) = -i\nu + g(x_1 - [x_1]_{2\tau}, x_2 - [x_2]_{2\tau})
\end{align*}
If $(x_1, x_2) \in B'_i$ and $i$ is even:
\begin{align*}
    f(x_1, x_2) = -i\nu + g(x_2 - [x_2]_{2\tau}, x_1 - [x_1]_{2\tau})
\end{align*}
If $(x_1, x_2) \in B_n$:
\begin{align*}
    f(x_1, x_2) = -n\nu + L(x_1 - s_1)^2 + L(x_2 - s_2)^2
\end{align*}

It remains to construct our buffer function $g$ and constant $\nu$. We will use spline theory to construct our buffer function. 

\begin{lemma}[See \cite{dougherty1989nonnegativity}]
For a differentiable function $f$, given $f(y_0), f(y_1), f'(y_0), f'(y_1)$, the cubic Hermite interpolant defined by
\[
    p(y) = c_0 + c_1\delta_y + c_2\delta_y^2 + c_3\delta_y^3
\]
where
\begin{align*}
    c_0 &= f(y_0)\\
    c_1 &= f'(y_0)\\
    c_2 &= \frac{3S - f'(y_1) - 2f'(y_0)}{y_1 - y_0}\\
    c_3 &= -\frac{2S - f'(y_1) - f'(y_0)}{(y_1 - y_0)^2}\\
    \delta_y &= y - y_0\\
    S &= \frac{f(y_1) - f(y_0)}{y_1 - y_0}
\end{align*}
satisfies $p(y_0) = f(y_0), p(y_1) = f(y_1), p'(y_0) = f'(y_0), p'(y_1) = f'(y_1)$. 
\end{lemma}

\begin{lemma}
There exists univariate function $g_1, g_2$, such that if $g(x_1,x_2) = g_1(x_1) + g_2(x_1) x_2^2$, function $f$ is continous, differentiable, and admits Lipschitz gradient when $\nu = \frac{1}{4}L \tau^2 - g_1(2\tau)$. 
\end{lemma}
\begin{proof}
Define 
\[p(x) = \frac{1}{2}(\gamma - L)x^2 + (L\tau - 2\gamma\tau) x,
\]
\[g_1(x) = p(x) - p(\tau) - \frac{1}{4}\gamma \tau^2.
\] Then
\begin{align*}
    g_1(\tau) &=  - \frac{1}{4}\gamma \tau^2\\
    g_1(2\tau) &= \frac{1}{4}L\gamma^2 - \nu \\
    g'_1(\tau) &= p'(\tau) = -\gamma \tau\\
    g'_2(\tau) &= p'(2\tau) = -L\tau
\end{align*}

By previous theorem, if we define $g_2$ to be the cubic function
\begin{align*}
    g_2(x; c_1, c_2) = c_2 - \frac{10(c_1 - c_2)(x - 2\tau)^3}{\tau^3} -\frac{15(c_1 - c_2)(x - 2\tau)^4}{\tau^4} - \frac{6(c_1 - c_2)(x - 2\tau)^5}{\tau^5}
\end{align*}
Then $g_2(\tau) = c_1, g_2(2\tau) = c_2$. Further,
\[
    g'_2(x; c_1, c_2) = -\frac{30(c_1 - c_2)(x - 2\tau)^2 (x - \tau)^2}{\tau^5}
\]

Thus in odd-to-even buffers when $x_1 < s_1$, we can choose
\[
    g_2(x) = g_2(x; L, L_2)
\]
when $x_1 \ge s_1$, we can choose
\[
    g_2(x) = g_2(x; L, -\gamma)
\]
and similarly for even-to-odd buffers. 

Finally it will easy to verify our construction of buffer function guarantees smoothness conditions on $f$. 
\end{proof}

\subsection{Proof of Main Theorem}

We will first define some notations that will be used throughout our analysis. 
There will be four main block types, which we refer to as odd, even, odd-to-even buffer, even-to-odd buffer, together with a final block. An odd block with the odd-to-even buffer on the right will be called the neighborhood of this odd block, and similarly an even block with the even-to-odd buffer on top of it will be called the neighborhood of this even block. Define $t_i$ to the the number of iterations spent in the block with the $i$-th saddle point, and $t'_i$ to be the number of iterations spent in the buffer block immediately following it. Further, define $T_i$ to be the first time gradient descent iterates escape the neighborhood of the $i$-th saddle point. Then
\[
    T_i + t_{i+1} + t'_{i+1} = T_{i+1}
\]
We use $d_1^t = |x_1^t - s_1|$ to denote the distance from $x_1$ to the central line along $x_2$ axis of the block at iteration $t$, and similarly for $d_2^t$. We will use a random initialization inside the first block, such that
\begin{align}
    d^0_1 \le \frac{\tau}{2e^2} \label{initialization}
\end{align}
We choose $\eta = \frac{1}{4L}$. Note that $\eta = \frac{1}{L_2}$, which combats the problem when some iterate ends up on the `wrong side of the block'. Too see this, suppose some iterate falls in the left half of some odd block, after a single update the $x_1$ coordinate will flip around the $x_1 = s_1$ vertical line, and proceed along the escape path. 

\begin{lemma}
For all $i$, $t'_i \le \frac{1}{\eta \gamma}$. 
\end{lemma}
\begin{proof}
Suppose $x^t$ is in some odd-to-even buffer. By our construction of function $g$, $\frac{\partial g}{\partial x_1} \le -\gamma\tau$, thus at each step, $x_1$ moves to the right by at least $\eta\gamma\tau$. The width of the block is $\tau$, thus gradient descent spends at most
\[
    \frac{\tau}{\eta\gamma\tau} = \frac{1}{\eta\gamma}
\]
iterations inside this block. A similar argument goes for even-to-odd buffers. 
\end{proof}

\begin{lemma}
Suppose we follow the initialization in~(\ref{initialization}). Then Gradient descent iterates stay in region $D$. 
\end{lemma}
\begin{proof}
We will use induction to prove this fact. 

$t_1$ must satisfy
\[
    d_1^0 (1 + 2\eta\gamma)^{t_1} > \frac{\tau}{2}
\]
By simple algebra, 
\[
    t_1 \ge \frac{1}{2\eta\gamma}\log(\frac{\tau / 2}{d_1^0}) \ge t_1'
\]

The initial point clearly belongs in $D$. Suppose $x^t \in B_1$, then $d_2$ is shrinking by a multiplicative factor, and $d_1$ is growing by a factor of $(1 + 2\eta\gamma) < 2$. Thus in the next iterate $x$ either stays in this block or escapes to the block immediately following it. 

Now suppose $x^t$ were in some odd-to-even buffer $B'_1$. Since $\frac{\partial g}{\partial x_1} > -L\tau$, 
\[
    x^{t+1}_1 \le x^{t}_1 +\eta L \tau \le x^t_1 + \tau. 
\]

After $t_1$ iterations from initialization, $d_2$ is no larger than $(1 - 2\eta L)^{t_1}$. We also know $\frac{\partial g}{\partial x_2} \ge -2\gamma (x_2 - s_2)$, thus inside the buffer $d_2$ grows by at most $(1 + 2\eta\gamma)$ per iteration. 
\begin{align*}
    d_2^{t_1 + t'_1 + 1} \le \frac{\tau}{2}(1 - 2\eta L)^{t_1} (1 + 2\eta \gamma)^{t'_1} &\le \frac{\tau}{2e^2}
\end{align*}
Note that this has exactly the same initialization condition as when gradient descent was initialized in the first block (with two coordinate flipped), and therefore we can apply induction. 
\end{proof}

\begin{theorem}
Let $f$ be as previously defined with $L \ge 2\gamma$. If $x^0$ is initialized inside first block and satisfies~(\ref{initialization}), then gradient descent with stepsize $\eta = \frac{1}{4L}$ takes $\Omega((\frac{L}{\gamma})^n)$ iterates to find the local minima. 
\end{theorem}

\begin{proof}
Consider iterates when $x_2$ is trying to escape the $i$-th even block (or equivalently the block with $2i$-th saddle point), denote the saddle position in this block to be $(s_1, s_2)$. Let $d^t$ be the distance from $x_2^t$ to $s_2$ at the $t$-th iteration: $d^t = |x_2^t  - s_2|$. 

In the $i$-th odd block, when $T_{2i-2} < t < T_{2i-2} + t_{2i-1}$, $d$ shrinks by a multiplicative factor of $(1 - 2\eta L)$ at each iteration. Mathematically, $x_2$ obeys gradient descent iterates
\[
    d^{T_{2i-2} + t_{2i-1}} = d^{T_{2i-2}} (1 - 2\eta L) ^{t_{2i-1}}. 
\]
In the odd-to-even buffer immediately following this odd block, we can bound $d^t$ as 
\[
    d^{T_{2i-1}} < d^{T_{2i-2} + t_{2i-1}}(1 + 2\eta\gamma)^{t'_{2i-1}}
\]
that is, $d$ grows no larger by a factor of $(1 + 2\eta\gamma)$ at each iteration. 

And finally in the $i$-th even block, $d$ grows by a multiplicative factor of $(1 + 2\eta \gamma)$ at each iteration, or
\[
    d^{T_{2i-1} + t_{2i}} = d^{T_{2i-1}} (1 + 2\eta\gamma) ^{t_{2i}}. 
\]

Since $d^{T_{2i - 1} + t_{2i}} > \frac{\tau}{2}$, we must have
\begin{align*}
    d^{T_{2i-2}} (1 - 2\eta L) ^{t_{2i-1}}(1 + 2\eta\gamma)^{t'_{2i-1}}(1 + 2\eta\gamma)^{t_{2i}} > \frac{\tau}{2}
\end{align*}
taking logarithms,
\[
    t_{2i} > \frac{L}{\gamma}t_{2i-1} - t'_{2i-1} > \frac{L}{\gamma} (t_{2i-1} - 4)
\]
Then it is not hard to show
\[
    t_{2i} - \frac{4L}{L-\gamma} > \frac{L}{\gamma}(t_{2i-1} - \frac{4L}{L - \gamma})
\]
and $t_1 > \frac{4L}{\gamma} \ge \frac{4L}{L - \gamma}$. 

Thus for gradient descent to escape to the global minimum, it would take at least $\Omega((\frac{L}{\gamma})^n)$ iterations. 
\end{proof}

\section{Experiments}
In this section we report our empirical findings through simulations, see Figure~\ref{fig:experiments}. In our experiments the objective function is as defined previously. In simulations we fix the number of saddle points $n$ to be 9 and used two values of $L$: $L = 1$ and $L = 1.5$. For the stochastic version of gradient descent, at each descent step, we add a independent gaussian noise to each direction with mean 0 and variance 0.1. 

We make three remarks. First, we observe that the performance of gradient descent aligns with our theoretical finding: the number of iterations needed to escape saddle points is growing by a multiplicative factor each time, and the larger the ratio $\frac{L}{\gamma}$, the larger this multiplicative factor is. Second, we notice that gradient descent eventually gets stuck at some saddle point. This is because the gradient may becoming vanishingly small during iterations. 
Thirdly, we note that SGD finds global minima efficiently without getting stuck. 
\begin{figure}
    \centering
    \includegraphics[scale = 0.65]{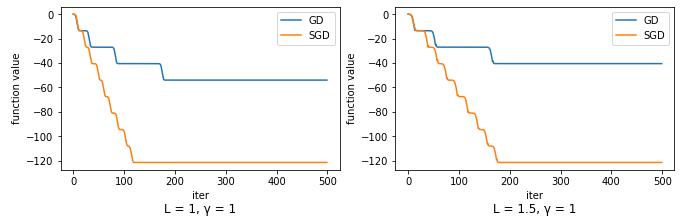}
    \caption{Performance of GD and SGD}
    \label{fig:experiments}
\end{figure}

\section{Conclusion}
In this paper we showed under a fairly reasonable initialization scheme with a non-pathological function of two variables, gradient descent might take exponential time in the number of saddle points to find a local minima. Further, we conducted empirical studies and verified our theoretical result. Our experiments also showed a noisy gradient descent algorithm is sufficient to find the global minima efficiently. Our result provide evidence that randomization and noise is a powerful tool to escape saddle points. 

The construction of our function does not allow second order derivatives. More careful treatment is needed to construct functions with well-defined Hessians in which gradient descent fails to escape saddle points efficiently. Another future direction might be to investigate negative results for gradient descent with momentum. Namely, does there exists some non-pathological function, such that accelerated gradient methods fail to converge efficiently?

\bibliographystyle{unsrt}
\bibliography{ref}

\end{document}